\documentclass[10pt]{article}
\usepackage[margin=1in]{geometry}
\usepackage[T1]{fontenc}
\usepackage{mathptmx}
\usepackage[utf8]{inputenc}
\usepackage{amsmath,amssymb,amsthm}
\usepackage{booktabs}
\usepackage{graphicx}
\usepackage{tikz}
\usetikzlibrary{arrows.meta,positioning,calc,fit}
\usepackage{pgfplots}
\pgfplotsset{compat=1.16}
\usepackage{algorithm}
\usepackage{algpseudocode}
\usepackage[numbers,sort&compress]{natbib}
\usepackage{microtype}
\usepackage[colorlinks=true,linkcolor=blue!50!black,citecolor=blue!50!black,urlcolor=blue!50!black]{hyperref}
\usepackage[capitalise]{cleveref}
\usepackage{caption}
\newtheorem{proposition}{Proposition}

\title{\textbf{Decode Late, Multiply Once:\\ Rethinking the Dot Product for Quantized Weights}}
\author{Jonathan Tapson\\ BrainChip Inc, Laguna Hills, CA, USA\\ \texttt{jtapson@brainchip.com}}
\date{October 2026}

\begin{document}
\maketitle

\begin{abstract}
Large neural networks, and large language models in particular, are increasingly deployed with very low-resolution weights (1--4 bits) while activations remain at 16 or 32 bits. We show that the dot products at the heart of such workloads can be reorganized so that elementwise multiplications are almost entirely removed. For a vector of length $n$ and weights quantized to $k$ bits, a conventional dot product needs $n$ weight decodes, $n$ multiplications and $n$ additions. In the \emph{add-before-multiply} (``scatter-accumulate'') refactoring, each weight code is used directly as an address offset that selects one of $2^k$ accumulators; each activation is added into the selected accumulator, and the $2^k$ accumulated sums are decoded and multiplied only once, at the end, in a fused decode-and-multiply step. The cost becomes $n+2^k$ additions (the $2^k$ products must still be summed), $2^k$ decodes and $2^k$ multiplications, and the decoded weight vector is never formed, which also shrinks local weight memory by up to $w_d/k$ for $w_d$-bit decoded weights. The reformulation is exact for integer arithmetic, is agnostic to integer or floating-point operands, accommodates arbitrary (including nonlinear) decoding, and extends to matrix and tensor products. We derive the method, give a cost model and worked example, and describe a hardware realization: a scatter-accumulate processing element (SA-PE) that replaces the MAC unit of a weight-stationary systolic array, with a four-phase load/stream/reduce/drain schedule in which a small, time-shared multiplier is used only $2^k$ times per tile. We give a local-memory model, report an RTL implementation with FPGA-flow sizing against a matched MAC array (the unshared SA-PE is larger than a MAC PE; the saving appears when the reduction multiplier is shared), report numerical-accuracy and functional experiments on software models, and discuss limitations and related prior work on weight sharing and lookup-based kernels.
\end{abstract}

\section{Introduction}
A recent development in AI is the use of low-resolution (1--4 bit) weights in very large neural networks, such as the foundation models behind modern LLMs~\cite{ma2024era,han2016deep}. The inference throughput of these models is generally limited by the rate at which weights can be loaded from memory, so low-resolution weights allow more weights to be delivered in a given time at fixed memory bandwidth.

The resulting workload is one in which a processing unit mostly multiplies a low-resolution weight (say 1--4 bits) by a high-resolution activation or other data variable (say 16 or 32 bits). These calculations normally take the form of large matrix or tensor multiplications, at the core of which are dot products: two vectors, which may be rows or columns of a matrix or tensor, are combined to produce a scalar. When optimizing matrix computation in a CPU, GPU or NPU, the best results require the minimum movement of data, the minimum use of local memory, and the minimum number of computations. Multipliers are the most expensive of the arithmetic units involved~\cite{horowitz2014computing}, yet in a low-bit-weight setting most of their capability is wasted: a $k$-bit weight takes only $2^k$ distinct values, so a long dot product multiplies the activations by the same few numbers over and over.

This paper describes a way of computing such dot products that exploits this repetition. The method works whether the weights are integer or floating-point encoded and whether the data variables are integer or floating-point encoded. A conventional dot product between a weight vector of resolution $k$ bits and a data vector of arbitrary resolution, each of length $n$, requires $n$ multiply-accumulate (MAC) operations, i.e.\ $n$ multiplications and $n$ additions. The new method decodes late: activations are summed by code first, and decoding and multiplication are fused into one step per bin, so it requires $2^k$ decodes, $2^k$ multiplications and $n+2^k$ additions (the $2^k$ products must still be summed). Since $k$ is usually 1--4 and $n$ may exceed 1000, the number of multiplications falls by $n/2^k$, which can exceed 100$\times$; there is a gain for any computation in which $2^k\ll n$. Because the $n$ additions remain, the achievable energy gain is bounded by the cost ratio of a decode--multiply--accumulate to a plain accumulate; we quantify this in \cref{sec:cost}.

\paragraph{Relation to prior work and contributions.} Related ideas exist. Weight-shared CNN accelerators bin inputs by weight index and multiply the bins by shared weights~\cite{garland2018low,garland2016low}; UCNN factorizes dot products over repeated weights~\cite{hegde2018ucnn}; lookup-table kernels such as LUT-GEMM and T-MAC replace multiplications in low-bit LLM inference~\cite{park2022lutgemm,wei2024tmac}; and ternary models such as BitNet b1.58 reduce multiplication to addition by construction~\cite{ma2024era}. We do not claim the identity is new. Our contributions are (i) extending it to floating-point operands and group-scaled decoders, stated for matrix and tensor products, and, as we read~\cite{hegde2018ucnn}, without the weight sorting and index metadata that UCNN uses to exploit repeated weights, and beyond the fixed-point CNN setting of~\cite{garland2018low,garland2016low} (\cref{sec:method,sec:cost,sec:example,sec:matrix}); (ii) integrating a scatter-accumulate processing element (SA-PE) as a drop-in node of an otherwise unchanged weight-stationary systolic array, with its load/stream/reduce/drain schedule, time-shared reduction engine, and closed-form sizing and local-memory models (\cref{sec:hw,sec:mem}); and (iii) validation of exactness, floating-point rounding error and operation counts, plus RTL sizing against a matched MAC array (\cref{sec:rtl,sec:eval}).

\section{Background and Notation}
\label{sec:background}
The dot product of two vectors $\mathbf a$ and $\mathbf b$ of length $n$ is
\begin{equation}
\mathbf a\cdot\mathbf b=\sum_{i=1}^{n}a_i b_i = a_1b_1+a_2b_2+\cdots+a_nb_n .
\label{eq:dot}
\end{equation}
When the elements of $\mathbf a$ have a low resolution of $k$ bits, they take only $2^k$ distinct values, which we write as integer \emph{weight codes}
$a_i\in\{0,1,\dots,2^k-1\}$.
Many AI algorithms exploit quantized weights by scaling them after they are loaded, so that each code is mapped to a more precise value; the mapping may be linear (scale and offset), a nonlinear function, or an arbitrary table. We represent it by a \emph{decoding vector} $\mathbf c=(c_0,\dots,c_{2^k-1})$ of precise coefficients: a code $v$ is replaced by $c_v$. The mapping may be computed as $c_v=f(v)$ for a linear or nonlinear decoding function $f$, or stored as a lookup table addressed by the code. Any such mapping is compatible with the method below.

The conventional computation begins with a decoding step that turns $\mathbf a$ into a vector $\mathbf d$ of precise weights, $d_i=c_{a_i}$, and then evaluates
\begin{equation}
\mathbf d\cdot\mathbf b=\sum_{i=1}^{n} d_i b_i ,
\label{eq:conv}
\end{equation}
requiring $n$ decode (lookup) operations, $n$ multiplications, $n$ additions, and storage for $\mathbf d$ or a pipeline that produces it on the fly. We call $\mathbf b$ the \emph{activations}, whatever its role in the network.

\section{The Add-Before-Multiply Method}
\label{sec:method}
\subsection{Refactoring}
All the elements of $\mathbf a$ that share a code are multiplied by the same coefficient, so the corresponding elements of $\mathbf b$ can be added together first and multiplied once. Define the \emph{accumulated group values} (bin sums)
\begin{equation}
q_v=\sum_{i:\,a_i=v} b_i,\qquad v=0,1,\dots,2^k-1,
\label{eq:bins}
\end{equation}
with the empty sum equal to zero.
\begin{proposition}
For all $\mathbf a,\mathbf b,\mathbf c$ as above, $\displaystyle \mathbf d\cdot\mathbf b=\sum_{v=0}^{2^k-1}c_v\,q_v$.
\end{proposition}
\begin{proof}
Partition $\{1,\dots,n\}$ by the value of $a_i$. Then $\sum_i c_{a_i}b_i=\sum_{v}\sum_{i:a_i=v}c_v b_i=\sum_v c_v\sum_{i:a_i=v}b_i=\sum_v c_vq_v$ by the distributive law.
\end{proof}
Written out in full,
\begin{equation}
\mathbf d\cdot\mathbf b=c_0\!\!\sum_{i:\,a_i=0}\!\! b_i+c_1\!\!\sum_{i:\,a_i=1}\!\! b_i+\cdots+c_{2^k-1}\!\!\!\!\sum_{i:\,a_i=2^k-1}\!\!\!\! b_i .
\label{eq:new}
\end{equation}
The number of multiplications and of decoding operations drops from $n$ to at most $2^k$ (fewer if only some codes occur in the vector and empty bins are skipped). The number of additions stays at $n$, plus at most $2^k-1$ to combine the weighted bins. What is introduced is the need to \emph{classify} the elements of $\mathbf b$ by the corresponding code.

\subsection{Weight codes as addresses}
One way to classify is to expand each $k$-bit code into a $2^k$-bit one-hot mask that gates the activation into the matching accumulator. This is simple, but costs one mask-generation operation per element and a mask-directed selection network. It is unnecessary: the code can be used directly as a binary address, or as an offset from a base address, of the accumulator in which its activation is summed. For $k=2$ there are four codes and four bin sums, held in four registers at $\mathrm{base}+\{0,1,2,3\}$ and selected by $a_i$. The decoder can be a $k$-bit passthrough or a small adder; no comparator or classifier is involved. \Cref{alg:both} gives the conventional and new procedures, and \cref{fig:flow} shows the dataflow.

\begin{algorithm}[t]
\caption{Conventional vs.\ add-before-multiply dot product. Inputs: codes $a[1..n]$, data $b[1..n]$, decode table $c[0..2^k-1]$.}
\label{alg:both}
\begin{minipage}[t]{0.48\textwidth}
\textbf{Conventional}
\begin{algorithmic}[1]
\State $p\gets 0$
\For{$i=1$ \textbf{to} $n$} \Comment{parallelizable}
 \State $d\gets c[a[i]]$ \Comment{decode}
 \State $p\gets p + d\cdot b[i]$ \Comment{MAC}
\EndFor
\State \Return $p$
\end{algorithmic}
\end{minipage}\hfill
\begin{minipage}[t]{0.48\textwidth}
\textbf{Add-before-multiply}
\begin{algorithmic}[1]
\State $q[0..2^k-1]\gets 0$
\For{$i=1$ \textbf{to} $n$} \Comment{parallelizable}
 \State $q[a[i]]\gets q[a[i]]+b[i]$ \Comment{add at offset}
\EndFor
\State $p\gets 0$
\For{$v=0$ \textbf{to} $2^k-1$}
 \State $p\gets p + c[v]\cdot q[v]$
\EndFor
\State \Return $p$
\end{algorithmic}
\end{minipage}
\end{algorithm}

\begin{figure}[t]
\centering
\begin{tikzpicture}[scale=0.92,transform shape,
  blk/.style={draw,rounded corners=2pt,minimum height=7mm,align=center,font=\footnotesize,inner sep=3pt},
  arr/.style={-{Stealth[length=2mm]},thick},
  lab/.style={font=\scriptsize,midway,above,align=center}]
 \node[font=\small\bfseries,anchor=west] at (0,2.6) {(a) Conventional};
 \node[blk] (a1) at (0.9,1.6) {codes $a_i$\\$n\times k$ bit};
 \node[blk,right=15mm of a1] (dec) {decode\\$n$ lookups};
 \node[blk,right=15mm of dec] (mul) {multiply\\$n$ MULs};
 \node[blk,right=15mm of mul] (sum) {accumulate\\$n$ ADDs};
 \node[blk,below=6mm of mul] (b1) {data $b_i$\\$n\times 16$ bit};
 \draw[arr] (a1)--(dec); \draw[arr] (dec)--node[lab,yshift=-1pt]{$d_i$}(mul);
 \draw[arr] (mul)--(sum); \draw[arr] (b1)--(mul);
 \draw[arr] (sum)--++(1.0,0) node[right,font=\footnotesize]{$p$};
 \node[font=\small\bfseries,anchor=west] at (0,-1.35) {(b) Add-before-multiply};
 \node[blk] (a2) at (0.9,-2.4) {codes $a_i$\\as address};
 \node[blk,right=15mm of a2] (acc) {bin accumulate\\$n$ ADDs at $q[a_i]$};
 \node[blk,right=15mm of acc] (mul2) {decode $\times$ multiply\\$2^k$ MULs};
 \node[blk,right=15mm of mul2] (sum2) {combine\\$2^k$ ADDs};
 \node[blk,below=6mm of acc] (b2) {data $b_i$\\$n\times 16$ bit};
 \draw[arr] (a2)--(acc); \draw[arr] (b2)--(acc);
 \draw[arr] (acc)--node[lab]{$q_v$}(mul2);
 \draw[arr] (mul2)--(sum2);
 \draw[arr] (sum2)--++(1.0,0) node[right,font=\footnotesize]{$p$};
\end{tikzpicture}
\caption{Dataflow of the conventional and add-before-multiply dot products. In (b) the $n$-element decoded vector $\mathbf d$ and the $n$ products are never formed; only $2^k$ bin sums are held.}
\label{fig:flow}
\end{figure}

\subsection{Delayed decoding, and two corollaries}
Decoding is \emph{delayed}: the codes are used directly during the long accumulation loop, and the coefficients are fetched or computed only for the $2^k$ bins afterwards. The storage that would hold the high-precision $\mathbf d$ is saved; $\mathbf d$ is never needed. The coefficient may come from a lookup table, a linear decoding function, or a nonlinear decoding function, and the three give identical results for a given code set.

\paragraph{Group-wise scales.} Practical quantization schemes often give each group of $g$ consecutive weights its own decode vector (a block scale, and possibly an offset). The method applies per group, with $(n/g)\,2^k$ multiplications in total; the advantage then requires $2^k\ll g$ rather than $2^k\ll n$.

\paragraph{One bin is redundant.} Because $\sum_v q_v=T\equiv\sum_i b_i$ depends only on the activations, one bin (say $q_0$) can be recovered as $T-\sum_{v\ge1}q_v$. When one activation vector is multiplied by many weight rows, $T$ is computed once and shared. For $k=1$ this means only elements with $a_i=1$ need be accumulated.

\section{Cost Model}
\label{sec:cost}
Let $E_d$, $E_m$, $E_a$ denote the energy (or latency, or area-delay product) of one decode, one multiplication and one accumulation, and $E_o$ the cost of forming the $k$-bit address offset for a bin update (about a $k$-bit integer addition, often absorbed in the addressing hardware). \Cref{tab:ops} lists operation counts per dot product, including the one-hot variant.

\begin{table}[t]
\centering
\caption{Operation counts for one length-$n$ dot product with $k$-bit weight codes. Counts are upper bounds if empty bins are skipped.}
\label{tab:ops}
\begin{tabular}{lccc}
\toprule
 & Conventional & One-hot masked & Direct address offset\\
\midrule
Decode / mask generation & $n$ decodes & $n$ masks & $0$ (offset only)\\
Coefficient decodes & $n$ & $\le 2^k$ & $\le 2^k$\\
Multiplications & $n$ & $\le 2^k$ & $\le 2^k$\\
Accumulator updates & $n$ & $n$ + $\le 2^k$ & $n$ + $\le 2^k$\\
Address-offset ops ($k$-bit) & 0 & 0 & $n$\\
Decoded-weight storage & $n$ values & $\le2^k$ coefficients & $\le2^k$ coefficients\\
Accumulation locations & 1 & $\le 2^k$ & $\le 2^k$\\
\bottomrule
\end{tabular}
\end{table}

Taking the bin additions to cost the same as a conventional accumulation, the energy ratio is
\begin{equation}
G(n,k)=\frac{n\,(E_d+E_m+E_a)}{n\,(E_a+E_o)+2^k\,(E_d+E_m+E_a)}\;\xrightarrow{\;n\to\infty\;}\;\frac{E_d+E_m+E_a}{E_a+E_o}.
\label{eq:gain}
\end{equation}
With $\rho=(E_d+E_m)/E_a$ and $E_o$ neglected, $G=(1+\rho)/\big(1+2^k(1+\rho)/n\big)$. The multiplier count falls by exactly $n/2^k$, but the total arithmetic gain saturates at $1+\rho$. For floating-point operands a multiplier typically costs several times an adder~\cite{horowitz2014computing}, so $\rho$ between 2 and 8 is a reasonable range for exploration (\cref{fig:gain}); for example $G\approx3.0$ ($\rho=2$) to $8.7$ ($\rho=8$) at $n=1024$, $k=2$. The headline ``more than $100\times$'' therefore refers to the \emph{number of multiplications}; the energy of the arithmetic falls by a factor approaching $1+\rho$, with larger gains where multipliers or lookups are comparatively expensive (wide floating-point formats, nonlinear decoders) or where the removed decoded-weight storage and traffic are significant. The saving also appears as area and switching activity: the per-cycle datapath needs an adder and a small accumulator file rather than a multiplier.

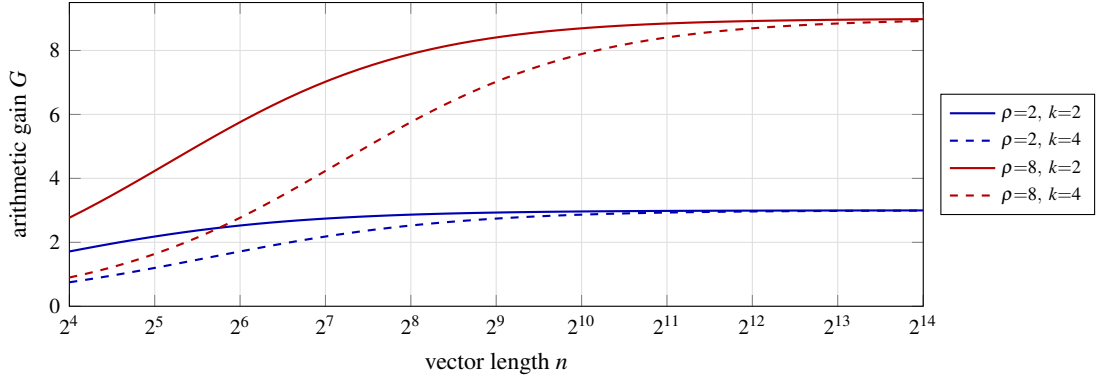
\begin{figure}[t]
\centering
\begin{tikzpicture}
\begin{axis}[width=0.78\textwidth,height=5.6cm,xmode=log,log basis x=2,
 xlabel={vector length $n$},ylabel={arithmetic gain $G$},xmin=16,xmax=16384,ymin=0,ymax=9.5,
 grid=both,grid style={gray!25},legend style={font=\scriptsize,at={(1.02,0.5)},anchor=west},
 tick label style={font=\footnotesize},label style={font=\small}]
\addplot[blue!70!black,thick,domain=16:16384,samples=60] {(1+2)/(1+4*(1+2)/x)};
\addplot[blue!70!black,thick,dashed,domain=16:16384,samples=60] {(1+2)/(1+16*(1+2)/x)};
\addplot[red!70!black,thick,domain=16:16384,samples=60] {(1+8)/(1+4*(1+8)/x)};
\addplot[red!70!black,thick,dashed,domain=16:16384,samples=60] {(1+8)/(1+16*(1+8)/x)};
\addlegendentry{$\rho{=}2,\ k{=}2$}\addlegendentry{$\rho{=}2,\ k{=}4$}\addlegendentry{$\rho{=}8,\ k{=}2$}\addlegendentry{$\rho{=}8,\ k{=}4$}
\end{axis}
\end{tikzpicture}
\caption{Arithmetic energy gain $G(n,k)$ from \cref{eq:gain} with $E_o=0$. The curves saturate at $1+\rho$ once $n\gg2^k(1+\rho)$.}
\label{fig:gain}
\end{figure}

\section{Worked Example}
\label{sec:example}
Take $n=10$, $k=2$, activations $\mathbf b=(.12,.34,.45,.56,.78,.90,.01,.23,.45,.67)$, codes $\mathbf a=(2,1,3,1,0,2,3,2,1,0)$ and a nonlinear decode vector $\mathbf c=(.06,.11,.37,.79)$. \emph{Conventionally}, the codes are decoded to $\mathbf d=(.37,.11,.79,.11,.06,.37,.79,.37,.11,.06)$, ten products are formed and summed: ten lookups, ten multiplications, ten additions, ten stored FP16 values, result $1.0614$.

\emph{With add-before-multiply}, \cref{tab:trace} shows the streaming phase: in each cycle one activation is added to the bin addressed by its code. The bins end at $q=(1.45,\,1.35,\,1.25,\,0.46)$, and the reduction gives
\[
1.45(.06)+1.35(.11)+1.25(.37)+0.46(.79)=.0870+.1485+.4625+.3634=1.0614,
\]
using ten bin updates, four decodes, four multiplications and four further additions, and holding only four values. The two results agree.

\begin{table}[t]
\centering
\caption{Stream-phase scatter trace and reduce-phase products for the worked example ($n=10$, $k=2$).}
\label{tab:trace}
\begin{tabular}{@{}lcccccccccc@{}}
\toprule
cycle $i$ & 0&1&2&3&4&5&6&7&8&9\\
\midrule
$b_i$ & .12&.34&.45&.56&.78&.90&.01&.23&.45&.67\\
code $a_i$ & 2&1&3&1&0&2&3&2&1&0\\
update & $q_2$ & $q_1$ & $q_3$ & $q_1$ & $q_0$ & $q_2$ & $q_3$ & $q_2$ & $q_1$ & $q_0$\\
\bottomrule
\end{tabular}

\vspace{4pt}
\begin{tabular}{@{}lcccc@{}}
\toprule
bin $v$ & 0 & 1 & 2 & 3\\
\midrule
accumulated $q_v$ & 1.45 & 1.35 & 1.25 & 0.46\\
coefficient $c_v$ & .06 & .11 & .37 & .79\\
product $c_vq_v$ & .0870 & .1485 & .4625 & .3634\\
\bottomrule
\end{tabular}
\end{table}

\section{Matrix and Tensor Products}
\label{sec:matrix}
Any matrix multiplication decomposes into dot products, so the method applies element by element. Let $A\in\mathbb R^{B\times K}$ be the activation matrix and $Q\in\{0,\dots,2^k-1\}^{K\times D}$ the weight-code matrix with decode vector $\mathbf c$. For each output element $y_{b,d}$ and each code $v$ there is an accumulator
\begin{equation}
Z_{b,d,v}=\sum_{\kappa:\,Q_{\kappa,d}=v}A_{b,\kappa},\qquad y_{b,d}=\sum_{v=0}^{2^k-1}c_v\,Z_{b,d,v},
\label{eq:matmul}
\end{equation}
i.e.\ $A_{b,\kappa}$ is added into $Z_{b,d,v}$ whenever $Q_{\kappa,d}=v$, and the decode and multiply happen once per output element, after the $K$-element accumulation. The result $Y=AQ'$ equals the product with the decoded matrix $Q'_{\kappa,d}=c_{Q_{\kappa,d}}$, but $Q'$ is never formed. The same holds for tensor contractions: the selected numerical values and codes participating in one output element are accumulated by code. A matrix--vector product with $D\times K$ weight codes needs $DK$ accumulations and $D\,2^k$ multiplications. Per output element the bin storage is $2^k$ accumulators, so the aggregate storage scales with the number of output elements processed concurrently---the design trade-off addressed by the weight-stationary schedule below.

\section{Hardware Realization}
\label{sec:hw}
\subsection{The scatter-accumulate processing element}
A scatter-accumulate processing element (SA-PE, \cref{fig:sape}) replaces the per-cycle multiplication of a conventional MAC element by a per-cycle address-selected accumulation, and defers all multiplication until accumulation is complete. It comprises: a \emph{weight-code storage structure} holding the code set for the current weight-stationary period, optionally double-buffered so the next period's codes load while the current one runs; an \emph{address decoder} producing a selection signal from the selected code (a $k$-to-$2^k$ one-hot decode, or base-plus-offset address generation, in the extreme a $k$-bit passthrough); a \emph{write-select circuit} (decoded per-bin write enables, a demultiplexer, or a decoded register-file or SRAM write address) that steers the incoming activation to the selected bin; an \emph{adder} that adds it to the value held in that bin; a \emph{bin accumulator bank} of $2^k$ locations; a $2^k$-entry \emph{decode LUT} (or decoding-function evaluator); and a \emph{reduction unit} that forms $\sum_j c_j\,\mathrm{Bin}[j]$ and presents it as the PE's partial sum; it fuses decode and multiply, since the decoded coefficient $c_j$ feeds the multiplier directly and the operation runs only $2^k$ times per tile. Signed codes are handled entirely in the decode LUT, so the address path operates only on unsigned $k$-bit codes. In the streaming datapath each cycle performs one address-selected add; the multiplier is not in the loop.

\begin{figure}[t]
\centering
\begin{tikzpicture}[
  blk/.style={draw,rounded corners=2pt,minimum height=7mm,font=\footnotesize,align=center,inner sep=3pt},
  arr/.style={-{Stealth[length=2mm]},thick}]
 \node[blk] (wc) at (0,1.5) {weight-code\\storage};
 \node[blk,right=7mm of wc] (ad) {address\\decoder};
 \node[blk,right=7mm of ad] (ws) {write-select};
 \node[blk,right=7mm of ws] (add) {adder};
 \node[blk,minimum height=14mm,right=7mm of add] (bank) {bin bank\\$\mathrm{Bin}[0..2^k{-}1]$};
 \node[blk,below=9mm of ws] (act) {activation in};
 \node[blk,below=9mm of bank] (lut) {decode LUT\\$c[0..2^k{-}1]$};
 \node[blk,right=9mm of bank] (red) {reduction\\unit};
 \draw[arr] (wc)--(ad); \draw[arr] (ad)--(ws); \draw[arr] (ws)--(add); \draw[arr] (add)--(bank);
 \draw[arr] (act)--(ws);
 \draw[arr] (bank)--(red);
 \draw[arr] (lut.east)-|(red.south);
 \draw[arr] (bank.north) -- ++(0,5mm) -| node[pos=0.25,above,font=\scriptsize]{read-modify-write} (add.north);
 \draw[arr] (red.east)--++(8mm,0) node[right,font=\footnotesize]{partial sum};
\end{tikzpicture}
\caption{Scatter-accumulate processing element. During streaming, the stored code selects a bin and the adder updates it; during the reduce phase the reduction unit multiplies each bin by its decoded coefficient and sums.}
\label{fig:sape}
\end{figure}
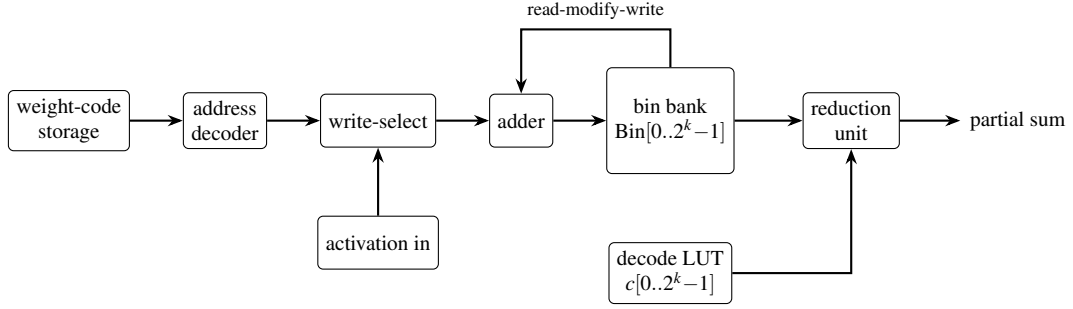

\subsection{Weight-stationary systolic array}
SA-PEs can replace the MAC unit at each node of a conventional weight-stationary systolic array (\cref{fig:array}); array-level control, tiling and interconnect of an existing design can be reused. Weight codes are loaded into the PEs once per weight tile and held stationary while activations are streamed. Each output element corresponds to a row of the array; the dot product is split into slices of $L$ elements processed by the $C$ PEs of that row ($n=CL$), each scatter-accumulating its slice of the activations into its own independently addressed bin bank. In each stream cycle a PE consumes one stationary code, read from its code store, and one activation; the gain needs $L\gg2^k$, since otherwise the bins collect too few terms to amortize the reduction. After a full pass of the activation slice, each PE holds $2^k$ partial bin sums and its reduction unit produces a partial dot product; a conventional row adder chain sums the partial results across the row to produce the output. With $H$ rows the array computes $H$ output elements per pass. A problem larger than the physical array is partitioned into weight tiles in the usual way, with row-partial results accumulated across tiles in an output accumulator.

\begin{figure}[t]
\centering
\begin{tikzpicture}[
  pe/.style={draw,rounded corners=2pt,minimum width=11mm,minimum height=8mm,font=\scriptsize,align=center},
  arr/.style={-{Stealth[length=1.8mm]},thick}]
 \foreach \r in {0,1,2}{
  \foreach \c in {0,1,2}{
    \node[pe] (p\r\c) at (\c*1.9,-\r*1.3) {SA-PE};
  }
  \node[pe,minimum width=9mm] (r\r) at (6.7,-\r*1.3) {red.};
  \draw[arr] (p\r0)--(p\r1); \draw[arr] (p\r1)--(p\r2);
  \node[font=\scriptsize,left=1mm of p\r0] {$y_{\r}$};
 }
 \foreach \c in {0,1,2}{
  \draw[arr] (\c*1.9,0.95)--(p0\c.north);
  \draw[thick] (p0\c.south)--(p1\c.north); \draw[thick] (p1\c.south)--(p2\c.north);
  \node[font=\scriptsize,above] at (\c*1.9,0.95) {$x$ slice $\c$};
 }
 \draw[arr] (p02)--(r0); \draw[arr] (p12)--(r1); \draw[arr] (p22)--(r2);
 \node[font=\scriptsize,align=center] at (7.9,-1.3) {row adder\\chain $\to y_m$};
\end{tikzpicture}
\caption{Weight-stationary array of SA-PEs. Activation slices are broadcast down columns; the PEs of a row hold codes for the same output element and their partial sums are combined by a row adder chain. (Schematic; array size is arbitrary.)}
\label{fig:array}
\end{figure}
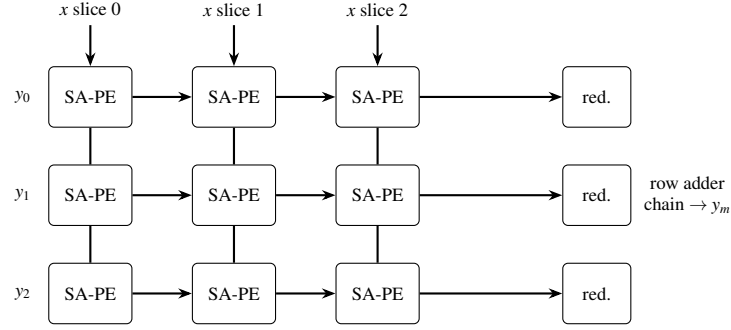

\subsection{Schedule}
A weight-stationary period has four phases (\cref{fig:phases}): \emph{load} (about one cycle: write the code set, initialize bins, load the decode LUT, reset counters), \emph{stream} ($L$ cycles, one scatter-accumulate per activation), \emph{reduce} ($2^k$ cycles of multiply-add in the reduction unit) and \emph{drain} ($C$ cycles to flush the $C$ partial sums of a row through the $C$-stage adder chain). The next tile's load can overlap the current tile's work through double buffering. Because $2^k\ll L$, load, reduce and drain are short compared with the stream phase, so throughput is set by one addition per activation per cycle; the period grows only by the $2^k$ reduce cycles relative to a conventional MAC pipeline (6\% for $k=2$, $L=64$, and 25\% for $k=4$), unless the reduce phase is overlapped with the next tile's stream, which requires double-buffered bins. The reduction multiplier runs only $2^k$ times per $L$ cycles, so it can be small, clocked slower than the stream path, or time-shared among all PEs of a row or column. A small controller finite-state machine with phase counters (a $\lceil\log_2 L\rceil$-bit stream counter and a $k$-bit reduce counter) sequences the phases through load, stream, reduce and drain states, asserting enables such as load-weights, init-bins, stream-enable and reduce-enable.

\begin{figure}[t]
\centering
\begin{tikzpicture}[
  ph/.style={draw,minimum height=6mm,font=\scriptsize,align=center,inner sep=1pt}]
 \node[ph,minimum width=6mm,fill=gray!15] (l) at (0.3,0) {L};
 \node[ph,minimum width=60mm,fill=blue!12,anchor=west] (s) at (0.6,0) {stream: $L$ cycles (1 add per cycle)};
 \node[ph,minimum width=9mm,fill=red!12,anchor=west] (r) at (6.6,0) {reduce};
 \node[ph,minimum width=11mm,fill=gray!15,anchor=west] (d) at (7.5,0) {drain};
 \node[ph,minimum width=6mm,fill=gray!15] (l2) at (6.2,-0.8) {L};
 \node[ph,minimum width=60mm,fill=blue!12,anchor=west] (s2) at (8.6,-0.8) {stream\ldots};
 \node[font=\scriptsize,anchor=east] at (-0.1,0) {tile $t$};
 \node[font=\scriptsize,anchor=east] at (-0.1,-0.8) {tile $t{+}1$};
 \node[font=\scriptsize,below] at (7.05,-0.15) {$2^k$};
\end{tikzpicture}
\caption{Phase schedule (not to scale). The load of the next tile overlaps the reduce and drain of the current one; the reduce phase lasts $2^k\ll L$ cycles.}
\label{fig:phases}
\end{figure}
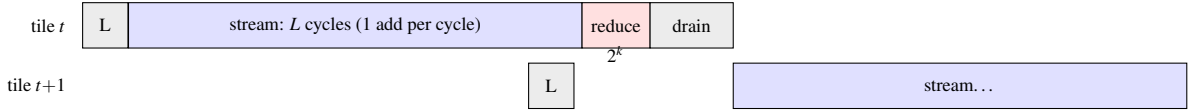

\subsection{Accumulator width and multi-precision design}
\emph{Bin width.} In the worst case all activations of a slice map to one code and sum into a single bin, so a fixed-point bin needs $w_{act}+\lceil\log_2 n_{\max}\rceil$ bits, where $n_{\max}$ is the number of activations accumulated into one bin (the slice length $L$ in an array); for 16-bit activations and $n_{\max}=1024$ that is ten guard bits. Alternatively the bins can be wider floating-point accumulators (e.g.\ FP32 for FP16 inputs), as is common practice in accelerators, which avoids bit growth.
\emph{Multi-precision.} A single SA-PE can support $k=1$ to $4$ (or more) by parameterizing the bin count and the address-decoder width. A design provisions the maximum bin count (16 for $k=4$); at lower $k$ the unused bins can be power- or clock-gated. The bin bank, whether registers or SRAM, is the dominant area of the PE, while the multiplier cost is amortized over the stream length regardless of $k$, so the design trade-off is set mainly by bin-bank size rather than arithmetic width.
\emph{Resource comparison.} \Cref{tab:pe} summarizes the per-PE comparison, and \cref{tab:red} gives the multiplication reduction $n/2^k$.

\begin{table}[t]
\centering
\caption{Per-PE resources for an $n$-length period: conventional MAC systolic PE versus SA-PE.}
\label{tab:pe}
\begin{tabular}{lll}
\toprule
 & Conventional MAC PE & SA-PE\\
\midrule
Multiplications per dot product & $n$ & $\le 2^k$, independent of $n$\\
Additions per dot product & $n$ accumulator updates & $n$ scatter updates + $\le 2^k$ reduce updates\\
Per-cycle streaming operation & 1 multiply + 1 accumulate & 1 address-selected add\\
Extra addressing logic & none & $k$-bit address decode / base+offset\\
Storage per PE & 1 accumulator & $2^k$ bins (width = data + $\lceil\log_2 n\rceil$)\\
Weight storage per PE & decoded weight, or $k$-bit + decode & $k$-bit codes only; decode LUT shared\\
\bottomrule
\end{tabular}
\end{table}

\begin{table}[t]
\centering
\caption{Multiplication reduction factor $n/2^k$ per PE per tile.}
\label{tab:red}
\begin{tabular}{rrrrr}
\toprule
$k$ & mults ($2^k$) & $n=256$ & $n=1024$ & $n=4096$\\
\midrule
1 & 2  & 128$\times$ & 512$\times$ & 2048$\times$\\
2 & 4  & 64$\times$ & 256$\times$ & 1024$\times$\\
3 & 8  & 32$\times$ & 128$\times$ & 512$\times$\\
4 & 16 & 16$\times$ & 64$\times$ & 256$\times$\\
\bottomrule
\end{tabular}
\end{table}

\subsection{Software and vector processors}
The same refactoring runs in software or on SIMD hardware: the loop of \cref{alg:both} is efficient on any core with cheap indexed read-modify-write to a small cache-resident table, and the address is formed as a base register plus the code offset. On a vector ISA with indexed stores (e.g.\ RISC-V \texttt{vsoxei}$*$), an implementation scatters each activation into a zero-filled $2^k\times \mathrm{VL}$ block at the row given by its code and then reduces the block. Because an indexed store overwrites rather than accumulates, the reduction then also adds the zeros ($2^k\,\mathrm{VL}$ lane-additions rather than $\mathrm{VL}$), so on such hardware the gain is confined to decode and multiplication; a scatter-\emph{accumulate} primitive or an accumulator file addressed by the code, as in the SA-PE, recovers the full operation count of \cref{tab:ops}.

\section{Local Memory Reduction}
\label{sec:mem}
Delayed decoding also shrinks the local weight memory. Consider one PE holding a weight-stationary tile of $L$ weights. Let $w_d$ be the decoded-weight width, $w_b$ the bin width, and $P$ the number of PEs that share one decode table; activation buffers are identical in all designs and are excluded. Three organizations are compared, with sizes in bits:
\begin{align}
M_A&=L\,w_d &&\text{(decode, then store the decoded weights)}\label{eq:MA}\\
M_B&=Lk+2^kw_d &&\text{(store codes, decode per element)}\label{eq:MB}\\
M_N&=Lk+2^k\!\left(w_d/P+w_b\right) &&\text{(scatter-accumulate)}\label{eq:MN}
\end{align}
Relative to a decoded tile the reduction is
\begin{equation}
R=\frac{M_A}{M_N}=\frac{L\,w_d}{Lk+2^k(w_d/P+w_b)}\;\xrightarrow{\;L\to\infty\;}\;\frac{w_d}{k},
\label{eq:mem}
\end{equation}
and $M_N<M_A$ once $L>2^k(w_d/P+w_b)/(w_d-k)$. \Cref{tab:mem} gives values for $L=1024$, $w_d=16$, FP32 bins ($w_b=32$) and $P=1$, and \cref{fig:mem} shows $R(L)$. For a row of 32 PEs with 1024-weight slices, 64\,KiB of decoded weights becomes 8.8\,KiB for $k=2$ or 19\,KiB for $k=4$. Using fixed-point bins of $w_{act}+\lceil\log_2L\rceil=26$ bits instead of FP32, or sharing the decode table across 32 PEs, changes $R$ by at most about 5\%.

\begin{table}[t]
\centering
\caption{Local weight memory per PE ($L=1024$, $w_d=16$, $w_b=32$, $P=1$).}
\label{tab:mem}
\begin{tabular}{rrrrrr}
\toprule
$k$ & $M_A$ (B) & $M_B$ (B) & $M_N$ (B) & $R=M_A/M_N$ & $M_N/M_B$\\
\midrule
1 & 2048 & 132 & 140 & 14.6 & 1.06\\
2 & 2048 & 264 & 280 & 7.3 & 1.06\\
4 & 2048 & 544 & 608 & 3.4 & 1.12\\
\bottomrule
\end{tabular}
\end{table}

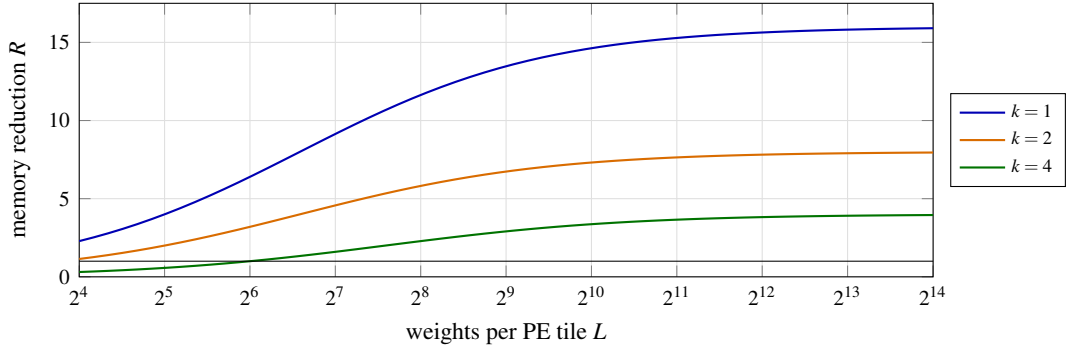
\begin{figure}[t]
\centering
\begin{tikzpicture}
\begin{axis}[width=0.78\textwidth,height=5.2cm,xmode=log,log basis x=2,xlabel={weights per PE tile $L$},ylabel={memory reduction $R$},xmin=16,xmax=16384,ymin=0,ymax=17.5,grid=both,grid style={gray!25},legend style={font=\scriptsize,at={(1.02,0.5)},anchor=west},tick label style={font=\footnotesize},label style={font=\small}]
\addplot[blue!70!black,thick,domain=16:16384,samples=60] {x*16/(x*1+2*(16+32))};
\addplot[orange!85!black,thick,domain=16:16384,samples=60] {x*16/(x*2+4*(16+32))};
\addplot[green!45!black,thick,domain=16:16384,samples=60] {x*16/(x*4+16*(16+32))};
\addplot[black,thin] coordinates {(16,1) (16384,1)};
\addlegendentry{$k=1$}\addlegendentry{$k=2$}\addlegendentry{$k=4$}
\end{axis}
\end{tikzpicture}
\caption{Memory reduction $R(L)$ from \cref{eq:mem} for $w_d=16$, $w_b=32$, $P=1$. The curves approach $w_d/k=16$, $8$, $4$; the thin line marks $R=1$.}
\label{fig:mem}
\end{figure}

The benefit depends on the baseline. It applies when the conventional design holds decoded weights locally ($M_A$). If the conventional design instead keeps $k$-bit codes and decodes per element ($M_B$), scatter-accumulate needs 6--12\% \emph{more} memory for the bins, and its advantage is the removed per-element decode and multiply rather than memory. The saving also requires $L\gg2^k$; for $k=4$ it vanishes at $L=64$.

\section{RTL Implementation and Sizing}
\label{sec:rtl}
\paragraph{Design.} We implemented the SA-PE and an $H\times C$ array in synthesizable Verilog, parameterized by $k$, slice length $L$, activation width $w_{act}$, rows $H$ and columns $C$. Each PE holds an $L\times k$-bit weight-code memory, $2^k$ signed bins and a multiply-accumulator used only in the reduce phase (bin $\times$ decode value, accumulated into a partial sum). A bin must hold the sum of $L$ activations and so needs at least $w_{act}+\lceil\log_2L\rceil$ bits (22 for $w_{act}=16$, $L=64$); the RTL adds one guard bit, using $\lceil\log_2(L+1)\rceil$ (23 bits), which gives a 42-bit partial sum. One $2^k$-entry decode LUT is shared by the whole array and broadcast during the reduce phase. A six-state controller (idle, flush, stream, reduce, drain, done) with stream, reduce and drain counters sequences each tile, and a registered adder chain sums the $C$ PE partial sums of each row over $C$ drain cycles.

\paragraph{Verification and baseline.} A self-checking testbench loads random codes, a random decode table and random activations into a $4\times3$ array and compares each output with an integer decode-then-MAC reference (8 random tiles, no mismatches). As baseline we built a matched MAC array with identical interfaces, code memories, shared LUT and controller (minus the reduce state); its PE decodes and multiplies every element, and it passes the same test. Sharing the structure keeps the comparison confined to the PE datapath.

\paragraph{Per-PE sizing.} Both arrays were synthesized with Yosys 0.33 to Xilinx 7-series LUT6 and DSP48E1 primitives ($k=2$, $L=64$, $w_{act}=16$). \Cref{tab:rtl} gives per-PE results; the code memory (one RAM64M) is identical in the two designs and is omitted. The last two rows are projections obtained by removing the multiplier from the PE and adding a separately synthesized reduction engine amortized over $\sigma=16$ PEs.

\begin{table}[t]
\centering
\caption{Per-PE FPGA resources ($k=2$, $L=64$, $w_{act}=16$, Yosys 0.33). $\dagger$: projection with the reduction multiplier shared by $\sigma=16$ PEs.}
\label{tab:rtl}
\begin{tabular}{lrrr}
\toprule
PE & LUT & FF & DSP48\\
\midrule
MAC, multiplier in DSP & 55 & 39 & 1\\
MAC, multiplier in fabric & 819 & 39 & 0\\
SA-PE (RTL), multiplier in DSP & 161 & 134 & 1\\
SA-PE (RTL), multiplier in fabric & 1193 & 134 & 0\\
SA-PE shared$^\dagger$, DSP engine & 122 & 95 & 0.06\\
SA-PE shared$^\dagger$, fabric engine & 179 & 95 & 0\\
\bottomrule
\end{tabular}
\end{table}

\paragraph{Array sizing.} A $4\times3$ array ($L=64$, 12 PEs, DSP-mapped) needs 2368 LUTs, 1807 FFs and 12 DSP48 for the SA array, against 1335 LUTs, 719 FFs and 12 DSP48 for the MAC array.

\paragraph{Interpretation.} The unshared SA-PE is \emph{larger} than a MAC PE (161 versus 55 LUTs and 134 versus 39 FFs with a DSP multiplier) because of the bin bank (92 flip-flops here), the partial-sum register and the bin-select multiplexer, and because it still contains one multiplier. The saving appears only when that multiplier is removed from the PE. A PE without the multiplier needs 119 LUTs and 92 FFs, and a stand-alone reduction engine needs 42 LUTs, 42 FFs and one DSP48 (960 LUTs if the multiplier is built in fabric). Shared by $\sigma=16$ PEs the engine adds about 3 LUTs per PE, so DSP use falls from one per PE to $1/16$, and with fabric multipliers the per-PE cost falls from 819 LUTs (MAC) to 179, a factor of $4.6$.

\paragraph{Sharing limit.} Each PE needs only $2^k$ multiplications per $L$ cycles, so one multiplier can serve at most $\sigma=L/2^k$ PEs: 16 for $k=2$, $L=64$, and 4 for $k=4$, $L=64$. The implemented RTL keeps one multiplier per PE. Sharing additionally requires the reduction of tile $t$ to overlap the stream of tile $t{+}1$, which needs double-buffered bins (about 92 more flip-flops per PE); otherwise the reduce phase grows to $\sigma2^k$ cycles.

\paragraph{Dependence on $k$.} The bin bank grows with $2^k$. At $k=4$ the SA-PE needs 359 LUTs and 412 FFs against 119 LUTs and 39 FFs for the MAC PE, and the serial reduce phase costs $2^k$ cycles per tile: 6\% of a 64-cycle stream at $k=2$ and 25\% at $k=4$. The code memory per PE is $Lk=128$ bits for $k=2$, $L=64$, plus the bins, against 1024 bits for a decoded 16-bit tile.

\paragraph{Caveats.} These are figures from an open-source FPGA flow. They indicate relative sizing, not ASIC area, power or Vivado results; the sharing figures are projections from synthesized components rather than synthesis of a complete shared design; and the comparison is for the specific datapath widths and parameters listed.

\section{Evaluation by Software Models}
\label{sec:eval}
\paragraph{Numerical behaviour.} For integer activations and integer decoded weights the refactoring is exact, since integer addition is associative and the distributive law holds; we verified this on random 8-bit data ($n=1000$, $k=2$). For floating point, reordering the additions changes rounding but not its order of magnitude. We measured this with a model in which $c_v$ are sorted normal variates in FP16, activations are normal in FP16 and codes are uniform; 40 random trials per row, FP64 reference. \Cref{tab:err} reports the mean error relative to $\sum_i|d_ib_i|$ for (i) the conventional method with FP32 multiply-accumulate, (ii) the conventional method with a fully FP16 datapath, (iii) the new method with FP16 bins and (iv) the new method with FP32 bins.

\begin{table}[t]
\centering
\caption{Mean relative error (normalized by $\sum_i|d_ib_i|$) against an FP64 reference.}
\label{tab:err}
\begin{tabular}{rrcccc}
\toprule
$k$ & $n$ & Conv.\ FP32 acc. & Conv.\ FP16 & New, FP16 bins & New, FP32 bins\\
\midrule
1 & 256  & $1.9\!\times\!10^{-8}$ & $1.6\!\times\!10^{-4}$ & $1.1\!\times\!10^{-4}$ & $2.6\!\times\!10^{-9}$\\
1 & 4096 & $1.7\!\times\!10^{-8}$ & $1.3\!\times\!10^{-4}$ & $1.1\!\times\!10^{-4}$ & $1.1\!\times\!10^{-9}$\\
2 & 256  & $2.2\!\times\!10^{-8}$ & $2.1\!\times\!10^{-4}$ & $8.5\!\times\!10^{-5}$ & $3.7\!\times\!10^{-9}$\\
2 & 4096 & $1.8\!\times\!10^{-8}$ & $1.5\!\times\!10^{-4}$ & $7.3\!\times\!10^{-5}$ & $1.2\!\times\!10^{-9}$\\
4 & 256  & $1.7\!\times\!10^{-8}$ & $1.6\!\times\!10^{-4}$ & $7.9\!\times\!10^{-5}$ & $3.4\!\times\!10^{-9}$\\
4 & 4096 & $1.9\!\times\!10^{-8}$ & $1.8\!\times\!10^{-4}$ & $5.8\!\times\!10^{-5}$ & $1.2\!\times\!10^{-9}$\\
\bottomrule
\end{tabular}
\end{table}

At equal accumulator width the new method is no less accurate than the conventional one, and more accurate in the FP16 setting, since the per-element product rounding is removed. Because the bins perform only additions of activations they can be widened cheaply; with FP32 bins the error is roughly an order of magnitude below the FP32 conventional pipeline, since only about $n+2^k$ rather than $3n$ roundings occur.

\paragraph{Array-level functional check.} We also modelled the tiled weight-stationary organization of \cref{sec:hw}: 30 random GEMV configurations ($k\in[1,4]$, up to 32 output rows, $K$ up to 700, random slice lengths $L$ per PE) in FP64 gave a maximum absolute difference of $4.3\times10^{-14}$ from the reference $(c[Q])\,x$, integer-exact results in an INT64 test, and a multiplication count of exactly $H\lceil K/L\rceil 2^k$, confirming the reduction factors of \cref{tab:red} (e.g.\ a $4096\times4096$ GEMV with $k=4$ and 1024-element slices uses $64\times$ fewer multiplications).

These are software emulations on synthetic data, not measurements on silicon or trained networks; end-to-end model accuracy and energy should be confirmed on real hardware and models.

\section{Discussion}
\label{sec:discussion}
\paragraph{Relation to prior work.} Exploiting repeated weights to factor out multiplications is an established idea, and this paper should be read in that context. Weight sharing with codebooks, which creates the few-distinct-values structure used here, was popularized by Deep Compression~\cite{han2016deep}. Accumulate-then-multiply units for weight-shared CNN accelerators, which bin inputs by weight index and multiply the bins by the shared weights, were proposed by Garland and Gregg~\cite{garland2018low,garland2016low}, and UCNN~\cite{hegde2018ucnn} exploits weight repetition by factorizing dot products and sharing partial sums. A separate, active line of work replaces multiplications in low-bit LLM inference by lookup tables of precomputed activation combinations (LUT-GEMM~\cite{park2022lutgemm}, T-MAC~\cite{wei2024tmac}), and ternary-weight models such as BitNet b1.58~\cite{ma2024era} reduce multiplications to additions by construction. Relative to these, the formulation here is deliberately general: it handles integer or floating-point operands and linear or nonlinear decoders, uses the weight code itself as the accumulator address so that neither a one-hot mask nor a decoded weight vector is formed, delays decoding until after accumulation, and is packaged as a drop-in PE for weight-stationary arrays with a streaming schedule and multi-precision provisioning. We make no claim that the algebraic identity is new; a detailed comparison with these earlier proposals would be needed to establish precisely where the benefit lies for a given target.

\paragraph{When it helps, and when it does not.} The benefit requires $2^k\ll n$ (or $\ll g$ for group-wise scales), multiplication and decode to be a meaningful part of the cost, and hardware that performs address-selected accumulation cheaply. In memory-bandwidth-bound inference, which motivates low-bit weights, reduced arithmetic does not by itself raise throughput, but the saved multiplier area and switching power can be spent on wider on-chip weight buffers or more parallel PEs within the same budget. Costs include the bin bank (the dominant PE area, $2^k$ wide accumulators per PE), a read-after-write hazard when consecutive activations share a code (requiring forwarding or interleaved bin sets), and, for large $k$, diminishing returns as $2^k$ approaches $n$. The RTL results of \cref{sec:rtl} show the first cost concretely: an unshared SA-PE is larger than a MAC PE, so the area benefit depends on sharing the reduction multiplier and on a bin bank that is cheap relative to the multiplier it replaces. The local-memory saving of \cref{sec:mem} likewise applies only against a baseline that stores decoded weights.

\paragraph{Alternatives.} For linearly coded weights one can split $\mathbf a$ into bit planes and compute $k$ masked sums scaled by powers of two; this needs no multiplications but $k$ conditional passes over $\mathbf b$, versus one pass here, and does not handle nonlinear decoders without an extra step.

\section{Conclusion}
Refactoring a low-bit dot product so that activations are added into code-addressed accumulators before any multiplication reduces the multiplier and decode count from $n$ to at most $2^k$, eliminates the decoded weight vector, and keeps the addition count at $n+2^k$ rather than $n$. A scatter-accumulate PE realizes this in a weight-stationary systolic array with a $2^k$-cycle reduce phase per tile and a small time-shared multiplier, and an RTL implementation shows the sizing trade-offs against a matched MAC array. The arithmetic gain saturates at the cost ratio of a decode--multiply--accumulate to a plain accumulate, so the method is most attractive where multipliers and decoders are expensive and accumulators can be made wide cheaply. Next steps are ASIC synthesis of the SA-PE and array with a shared reduction engine, to replace the abstract energy model with measured area and power, and evaluation of end-to-end accuracy on quantized models.

\end{document}